\documentclass{article}
\usepackage[preprint]{neurips_2026}
\usepackage[utf8]{inputenc}
\usepackage[T1]{fontenc}
\usepackage[colorlinks=true,allcolors=blue]{hyperref}
\usepackage{url}
\usepackage{booktabs}
\usepackage{amsfonts}
\usepackage{amsthm}
\usepackage{amsmath}
\usepackage{amssymb}
\usepackage{nicefrac}
\usepackage{microtype}
\usepackage{xcolor}
\usepackage{graphicx}
\usepackage{subcaption}
\usepackage{multirow}
\usepackage{enumitem}
\usepackage{float}
\usepackage{wrapfig}
\setcitestyle{square,comma}

\newtheorem{definition}{Definition}
\newtheorem{theorem}{Theorem}
\newcommand{\norm}[1]{\left\|#1\right\|}

\title{Do Flatter Minima Drive Better Generalization? An Algorithmic Separation in Grokking}
\author{%
  Mohnish Harwani \\
  Purdue University \\
  \texttt{mharwan@purdue.edu}
}

\begin{document}
\maketitle

\begin{abstract}
Flat loss landscapes have long been linked to better generalization in neural networks. However, its role as a causal mechanism for generalization is less established. Grokking provides an unique testbed to understand this distinction: models are prone to fit observed data using non-generalizing structure and remain in that regime for prolonged periods, transitioning to generalization only under particular training conditions. In this work, we study whether flat loss landscapes can act as a driving mechanism in this transition. While recent work has argued for flatness as a necessary geometric condition for this transition, we find that biasing training toward flatter solutions using sharpness-aware minimization (SAM) is insufficient to reliably induce this transition, despite producing flatter solutions. However, when SAM is paired with mechanisms that drive generalization such as weight decay, an interesting property emerges: SAM can accelerate the transition to generalizing solutions by up to 4x at the epoch-level. We theoretically untangle this relationship between SAM and weight decay using a minimal interpolating two-layer ReLU model with both memorizing and generalizing solutions. We show that even in this simple setup, flatness alone cannot distinguish a memorizing solution from a generalizing one, while weight decay favors generalizing solutions. However, under a local stability analysis, there exists a window where a memorizing interpolant is locally stable under gradient descent but unstable under SAM in the low-norm regime, which can explain SAM's ability to accelerate this transition. Overall, our results provide a more interpretable account of the role of flatness in driving generalization, especially in settings where models are vulnerable to minimizing loss through learning non-generalizing structure.
\end{abstract}

\section{Introduction}
\label{sec:intro}
Distinguishing properties that correlate with generalization from those that causally drive it is especially important when learning from noisy, real-world data. In these settings, models can exploit sample-specific features and noise to achieve high performance on their training data, while failing to generalize on unseen data. An attractive mechanism for inducing generalization is biasing training towards flatter loss landscapes which have long been associated with generalization in artificial neural networks. Models that converge to flatter minima often exhibit better performance on unseen data \citep{hochreiter1997flat,keskar2017largebatch}, motivating optimization methods such as sharpness-aware minimization (SAM) \citep{foret2021sam} that explicitly bias training toward locally flat solutions at an additional computational cost \citep{andriushchenko2022towards,compagnoni2023sde,wen2023sharpness}. 

To understand whether flatness drives generalization in settings where memorizing solutions are particularly attractive, we study grokking, a unique phenomenon where models first reach near-zero training loss while remaining near chance on held-out data, and only after prolonged optimization abruptly transition to high test accuracy \citep{power2021grokking}. Unlike conventional training, where memorization and generalization often develop simultaneously, grokking exposes a temporally separated transition from a memorizing solution to a generalizing one. This makes it possible to algorithmically separate and track differing characteristics between memorizing and generalizing interpolants. Prior work has uncovered two leading mechanisms that can induce this transition from memorization to generalization. Mechanistic studies of modular arithmetic have established a decrease in parameter norm as central in inducing this transition from memorization to generalization, arguing that generalizing circuits become favorable once they are cheaper than memorizing alternatives \citep{varma2023explaining}. More recently, flatness has been proposed as a necessary geometric condition to induce this transition, \citep{han2025flatness} which stands as an alternative explanation. 

Our observations reconcile these two threads of prior research on mechanisms that induce generalization in grokking. We use SAM to directly test whether biasing optimization toward flatter solutions is sufficient to trigger the memorization-to-generalization transition. Across several architectures on modular arithmetic tasks, we find that it is not: SAM alone consistently produces substantially flatter solutions while models remain memorizing. These experiments provide direct counterexamples to the hypothesis that flatness alone selects generalizing solutions, challenging existing work. \citep{han2025flatness} However, when pairing SAM with weight decay, we observe that the system reaches generalizing regimes substantially earlier at an epoch level: up to 4x faster than weight decay alone. This acceleration increases with the strength of SAM's flatness regularization, until SAM can no longer meaningfully enforce flatness due to inaccuracy of the first-order Taylor estimate that it relies upon. Figure~\ref{fig:frontpage} summarizes this contrast. We explain these phenomena through an algorithmic separation: while SAM fails to distinguish memorizing and generalizing interpolants (whereas parameter norm does), SAM changes the local stability of memorizing solutions once weight decay has moved the system into the appropriate low-norm regime.

\begin{wrapfigure}{r}{0.65\textwidth}
\vspace{-10pt}
\centering
\includegraphics[width=\linewidth]{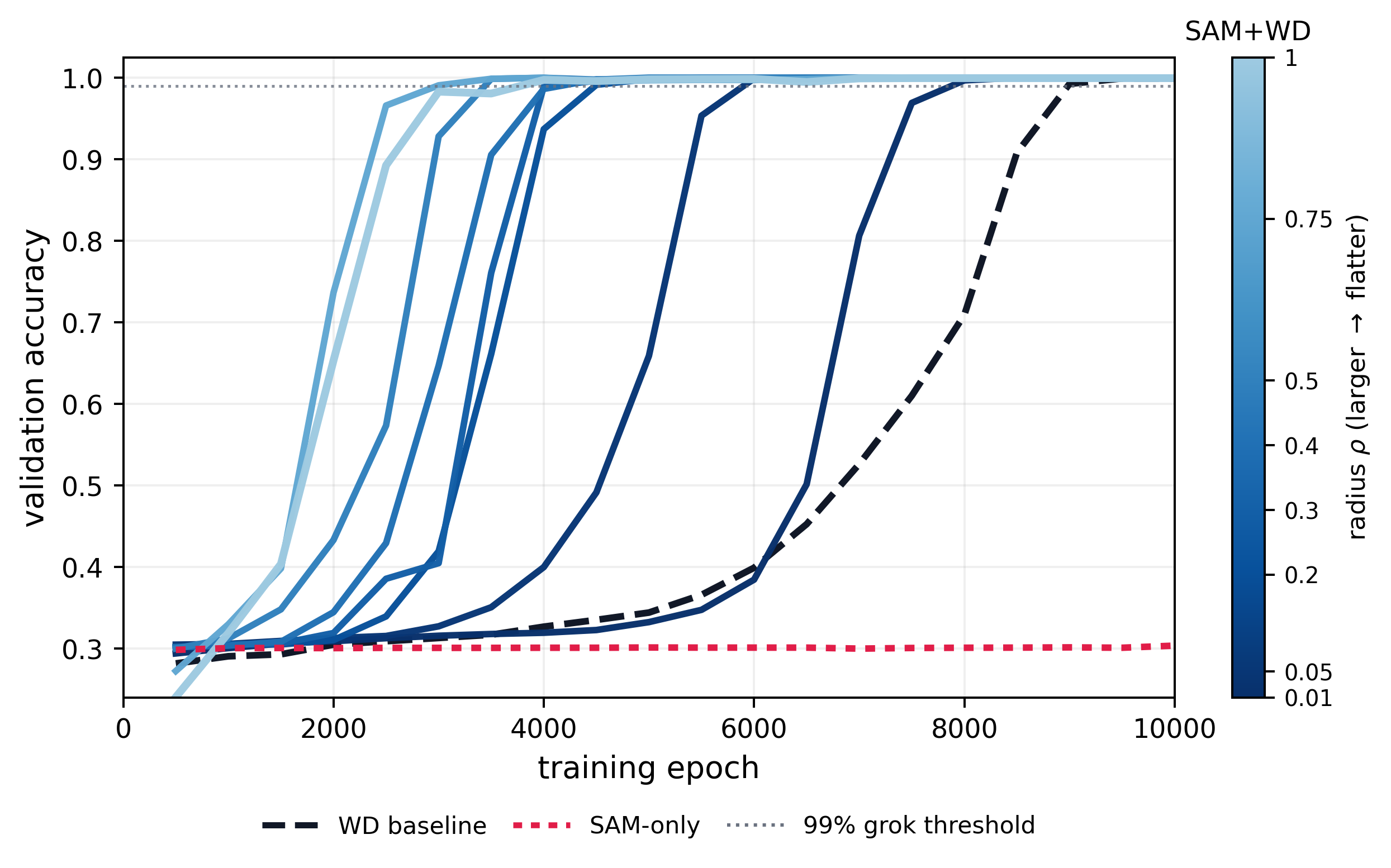}
\vspace{-10pt}
\caption{Flatness regularization accelerates generalization in grokking where grokking already occurs. It does not, however, trigger generalization alone. Lighter colors correspond to stronger regularization. Stronger regularization is linked with faster generalization. }
\vspace{-10pt}
\label{fig:frontpage}
\end{wrapfigure}

To develop a theoretical understanding of this algorithmic separation, we construct a minimal two-layer interpolating network that isolates the sufficient structural conditions under which this separation emerges. We explicitly construct two classes of local optima—\emph{memorizing} and \emph{generalizing} interpolants—which, by design, have identical trace flatness; thus flatness does not distinguish between them. At the same scale, however, generalizing interpolants use shared features rather than sample-specific neurons, and (for sufficiently large training set size) the norm of the generalizing solutions is strictly smaller than that of the memorizing solutions. Thus, incorporating weight-decay into optimization can trigger grokking as it leads to eventual decrease in the norm of the parameters. We then analyze local stability under gradient descent (GD) and SAM. Under the quadratic SAM approximation, SAM changes the `effective' local curvature, producing a regime in which a memorizing interpolant is stable under GD but unstable under SAM. This shows that SAM is not the selector of generalization: selection is driven by norm and weight decay \citep{loshchilov2019decoupled,nacson2024wd}. SAM instead conditionally accelerates escape from memorizing solutions faster than GD once the trajectory has entered low-norm regimes where memorizing solutions become fragile.

We make three key contributions in this work. First, we provide empirical counterexamples to flatness-as-selector explanations of grokking on interpretable, modular arithmetic tasks: SAM-only produces very flat models that remain memorizing, while WD and SAM+WD grok. Second, we prove an algorithmic separation between flatness, norm, and local stability in an interpolating ReLU model. 
Third, we extend our findings beyond modular arithmetic to show strikingly similar acceleration-flatness relationships in image classification and natural language processing tasks across several architectures, demonstrating the broader applicability of our work beyond interpretable settings. Together, these results highlight the interplay between loss landscape geometry (via flatness), model capacity (via norm), and generalization.

\section{Related Work}
\label{sec:related}

\noindent\textbf{Grokking and its mechanisms.}
\citet{power2021grokking} discovered grokking on modular arithmetic and noted its dependence on weight decay.
\citet{nanda2023progress} showed grokking corresponds to Fourier circuit formation and introduced progress measures including weight norm.
\citet{varma2023explaining} formalized this via a circuit efficiency framework: grokking occurs when $\norm{\mathbf{w}}$ falls below a threshold where generalizing circuits become cheaper than memorizing circuits.
\citet{liu2023omnigrok} identified a ``Goldilocks zone'' for weight norm in grokking.
More recently, \citet{zhang2025glass} studied grokking under WanD and reported generalizing arithmetic solutions at substantially higher parameter norms than the AdamW threshold we identify here, suggesting that the circuit-efficiency crossover norm can be optimizer-dependent.
We therefore do not claim a universal optimizer-independent norm law; our causal claim is scoped to grokking under the standard full-batch AdamW/SAM family.

\noindent\textbf{SAM and the flatness hypothesis.}
The hypothesis that flat minima generalize better has a long history~\citep{hochreiter1997flat, keskar2017largebatch}.
One popular optimization algorithm in this line of research is Sharpness-Aware Minimization (SAM), which minimizes maximum loss in the vicinity of the current minima \citep{foret2021sam}.  \citet{dinh2017sharp} showed that sharpness measures can be manipulated by reparametrization, and \citet{kaur2023maximum} demonstrated that $\lambda_{\max}$ can decrease under SAM without generalization gains.
Theoretical analyses have revealed that SAM's implicit bias differs from direct curvature minimization: \citet{andriushchenko2022towards} showed SAM biases toward $\ell_1$-norm minimization in diagonal linear networks, while \citet{compagnoni2023sde} showed via SDE analysis that SAM implicitly regularizes the gradient norm $\norm{\nabla\mathcal{L}}$ rather than the Hessian.
\citet{wen2023sharpness} argued that sharpness alone is neither necessary nor sufficient for generalization, which is consistent with the flat-but-memorizing counterexample we isolate here in grokking.

\noindent\textbf{Flatness in grokking.}
\citet{han2025flatness} observes that regularizing for \textit{sharper solutions} delays generalization, resulting in grokking-like phenomena. They conclude that flatness must, in turn, be a necessary prerequisite for generalization within grokking. We recreate their work on language and vision tasks, as well as more canonical grokking setups such as modular arithmetic, finding that 1) flatness can occur without generalization in grokking and 2) sharp solutions often generalize. Specifically, our work challenges the  necessity claim: while relative flatness can correlate with the onset of generalization, our results show that this correlation does not extend to flatness being a causal selector of this transition. Additionally, we refine their work to demonstrate specific ways flatness can be used to actually accelerate generalization (as opposed to using sharpness to delay it) which can explain the benefit of methods like SAM in practice. We distinguish a relationship between the extent of flatness and the speed of generalization in grokking. Lastly, we discuss how their proposed metric of relative flatness is not invariant under changes in confidence or scaling, which can potentially be misleading, motivating our work further. See Section~\ref{sec:han_bad}.

\noindent\textbf{Grokking beyond algorithmic data.} Rather than attributing delayed generalization solely to the discrete structure of algorithmic tasks, \citet{liu2023omnigrok} argue that grokking can arise from a geometric mismatch between the training and test loss landscapes. Grokking occurs when optimization first reaches a high-norm memorizing solution, after which regularization slowly moves the model toward the lower-norm region where test performance improves. By increasing initialization scale and controlling regularization, this work induces this transition on image, language, and molecular prediction tasks. We recreate their setup, training MLPs on MNIST and LSTMs on IMDb with larger initialization scales and weight decay, reinforcing the existence of a relationship between the extent of flatness and the speed of grokking in more practical training regimes, further demonstrating the broader applicability of our results.

\section{Preliminaries}
To enforce flatter loss landscapes and test our claim, we use SAM, as introduced by \citet{foret2021sam}. SAM relies on a first-order Taylor estimate of the Hessian to estimate the degree of flatness in the loss landscape, with a higher value of $\rho$ indicating an estimated flatter region, until $\rho$ becomes too large and the flatness can no longer be enforced due to the first-order Taylor estimate degrading in accuracy.

Specifically, consider a model $f$ : $X \rightarrow Y$ parameterized by a weight vector $w$ and a per-sample loss function $l$: $W \times X \times Y \rightarrow R_+$. Define training data S = \{($x_1$, $y_1$),..., ($x_n$, $y_n$)\} sampled i.i.d. from a data distribution D. The traditional training loss is defined as $L_S(w) = \sum_{i=1}^{n} l(y_i, f(x_i, w))/n$. SAM searches for a flatter minimum by combining the traditional loss with maximum loss in the vicinity defined by perturbation radius $\rho$. Formally, it is defined as follows: 

\begin{align*}
L_{SAM}(w) 
& = \min_w \max_{\|\epsilon(w)\|_2 \leq \rho}L_S(w+\epsilon(w)),\\
\epsilon^*(w) & \approx \rho \cdot sign \left( \nabla_w L_{S}(w) \right) \cdot \dfrac{\lvert \nabla_w L_{S}(w)\rvert}{\left\| \nabla_w L_{S}(w) \right\|_2}
\end{align*}

\section{Theoretical Motivation}
\label{sec:theory}

In this section, we construct and analyze a 2-layer ReLU network to motivate why flatness can accelerate but not induce generalization. We find that, in this minimal example, flatness by itself may not be able to act as a selector between memorizing and generalizing branches because these branches can have identical trace flatness while generalizing branches have strictly smaller parameter norms. We then analyze local stability under gradient descent (GD) and SAM, showing how in these solutions a memorizing interpolant can be unstable under SAM but not GD, indicating a potential avenue through which SAM can conditionally accelerate escape from memorizing solutions faster and thus accelerate the transition to generalizing solutions.

\subsection{Setup: Interpolating ReLU construction}
\label{subsec:relu_construction}

Let $
x\in\{-1,+1\}^d,
y(x)=x_1x_2,
$ and 
$S=\{(x_i,y_i)\}_{i=1}^n$ 
be a training set with distinct inputs and labels $y_i=y(x_i)$. We consider a two-layer ReLU network
$
f_\theta(x)
=
\sum_{j=1}^m a_j\sigma(w_j^\top x+b_j),
\text{ with }
\sigma(t)=\max\{t,0\}.
$
We use the squared training loss
$
\mathcal{L}_S(\theta)
=
\frac{1}{2n}\sum_{i=1}^n
(f_\theta(x_i)-y_i)^2.
$
At an interpolating differentiable solution, the Hessian is:
$
H(\theta)
=
\nabla^2\mathcal{L}_S(\theta)
=
\frac{1}{n}\sum_{i=1}^n g_i g_i^\top$, with $
g_i:=\nabla_\theta f_\theta(x_i)$.

\noindent\textbf{Memorizing interpolant.}
For each training sample $x_i$, define one neuron
$
w_i=r x_i,
b_i=-r(d-1),
a_i=\frac{y_i}{r}.
$
This neuron activates only on the full data point $x_i$, thus memorizing it. Thus, the resulting network interpolates the training set but outputs zero on hypercube points not in the training set.

\noindent\textbf{\((C,r)\)-generalizing interpolant.}
Fix $C\ge 2$. For each pattern
$
z=(z_1,\dots,z_C)\in\{-1,+1\}^C,
$
define one neuron
$
w_z=r\sum_{k=1}^C z_k e_k,
b_z=-r(C-1),
a_z=\frac{z_1z_2}{r}.
$
For an input $x$, exactly one neuron matching the pattern activates
$
z=(x_1,\dots,x_C).
$
The output is then
$
f_\theta(x)=x_1x_2.
$
Thus, multiple inputs can activate the same neuron, in contrast with the memorizing interpolant where each neuron is activated by exactly a single data point. Thus the \((C,r)\)-interpolant generalizes because it has \emph{learnt} a superset of the pattern necessary for generalization. Smaller $C$ corresponds to sparser learning. Further define $q_d(r):=r^2+\frac{d+1}{r^2}$. As we shall see in the sequel, $q_d$ measures the flatness of the interpolant, entirely dependent on $r$ and independent of $C$. Thus $(C,r)$ interpolants identify a set of solutions with varying complexity ($C$) and flatness (dependent on $r$). 

\subsection{Algorithmic separation of solutions}
\label{subsec:selection_stability}
We analyze selection among interpolating solutions through \emph{local stability} of these optima~\citep{chang2025unified,dexter2024precise}.
Since both memorizing and generalizing solutions can achieve zero training loss, the tool of local stability tells us which of them are stable under the optimizer dynamics. In other words, given certain local properties of the optimum, will an optimization algorithm (e.g. gradient descent or SAM) settle at this optimum (i.e., be stable) or escape from it (unstable).
We use flatness and norm to compare relative stability properties of SAM and GD on memorizing vs generalizing minima.

\begin{definition}[Linear stability under GD and SAM~\citep{chang2025unified}]
\label{def:linear_stability}
Let $\theta^\star$ be an interpolating solution with Hessian $H\succeq 0$. Under local quadratic approximation, GD with stepsize $\eta$ has linearized update
$
J_{\rm GD}=I-\eta H.
$
Under the quadratic approximation for SAM used by \citet{chang2025unified}, the effective curvature is
$
H_{\rm SAM}=H(I+\gamma H),
\gamma:=\frac{\rho}{\alpha}>0,
$
and the linearized update is
$
J_{\rm SAM}=I-\eta H(I+\gamma H).
$
\end{definition}

Under the localized stability analysis, we say $\theta^\star$ is linearly stable under a method if the corresponding update matrix has spectral radius strictly less than one, which ensures that $\theta^\star$ lies in an attractor basin from which the optimizer will not escape (is stable within). We now present our result.

\begin{theorem}[Algorithmic separation]
\label{thm:main_separation}
Consider the memorizing interpolant $\theta_{\rm mem}(r)$ and the $(C,r)$-generalizing interpolant $\theta_{C,r}$ defined in Section~\ref{subsec:relu_construction}. Then the following hold.

\begin{enumerate}
\item \textbf{Flatness does not separate.} Both interpolants have the same trace flatness:
    \[
    \operatorname{Tr}H(\theta_{\rm mem}(r))
    =
    \operatorname{Tr}H(\theta_{C,r})
    =
    q_d(r).
    \]
\item \textbf{Norm separates.}
For any fixed scale $r>0$ and sufficiently large $n>  \frac{2^C( r^2 (C+(C-1)^2)+r^{-2})}{ r^2(d+(d-1)^2)+r^{-2}},$ 
\[
\|\theta_{C,r}\|_2^2 < \|\theta_{\rm mem}(r)\|_2^2.
\]
Thus, among interpolating solutions at the same scale, the generalizing interpolant has strictly smaller parameter norm.
\item   \textbf{SAM accelerates within the low-norm regime.} There exist thresholds $r_{\rm low }>0, r_{\rm high}>0$ such that for any
$
r \in (r_{\rm low}, r_{\rm high})
$
the memorizing interpolant $\theta_{\rm mem}(r)$ is locally stable under GD but unstable under SAM.
\end{enumerate}
\end{theorem}

\noindent\textbf{Discussion.}
Theorem~\ref{thm:main_separation} separates three roles that are often conflated in explanations of grokking. First, flatness does not distinguish between memorizing and generalizing solutions: both interpolants have identical trace flatness, so flatness alone cannot act as a selector. Second, norm provides a clear separation: at the same scale, the generalizing interpolant uses shared features and therefore achieves strictly smaller parameter norm than the memorizing interpolant. This identifies weight decay, rather than flatness, as the mechanism that selects the generalizing branch.

Finally, the theorem clarifies the role of SAM. In high-norm regimes, memorizing solutions remain \emph{stable} under both GD and SAM i.e., GD and SAM can both get ``stuck" in these minima.  However, once weight decay reduces the norm into a low-norm regime, SAM destabilizes memorization earlier than GD, leading to faster escape from such minima. Thus SAM accelerates the transition to generalization, but only after the system has entered the regime selected by weight decay. While our theoretical result and empirical evaluation are in the context of norm and weight decay, we believe the norm here is just a proxy for model capacity in general. We leave these investigations as future work, and instead utilize this framework to motivate why this phenomenon occurs in grokking.

\section{Experiments}
\label{sec:experiments}

Motivated by our theoretical analysis, we examine the effect of SAM and weight decay on several datasets and architectures. We design two sets of experiments. First, we attempt to isolate causal mechanisms in generalization through modular arithmetic, a series of canonical grokking tasks where the epoch-level transition between memorizing and generalizing interpolants is large, and is well studied and more interpretable \citep{nanda2023progress}. In this setting, we demonstrate how flatness is insufficient to trigger generalization, while weight decay is. We also show how SAM paired with weight decay accelerates the grokking transition as $\rho$ increases, until the benefits begin to taper off as SAM is unable to enforce flatness at high values of $\rho$. 

Second, we show on vision and language tasks that demonstrate grokking-like behavior \citep{han2025flatness, liu2023omnigrok} and generalize consistently through standard training conditions, SAM maintains this trend of accelerating how fast generalizing solutions are reached at the epoch-level. While it is challenging to isolate causal mechanisms of generalization on these tasks, they support our claim that SAM accelerates the transition to a generalizing regime once sufficient conditions to trigger generalization are already present.

\subsection{Grokking on Modular Arithmetic}

Modular arithmetic tasks are particularly attractive for testing mechanisms that induce generalization, because the onset of generalization is well studied and linked to mechanisms such as the development of Fourier-like circuits \citep{nanda2023progress}. This allows us to design experiments that isolate what drives the transition to generalizing solutions. 

\begin{table}[H]
\centering
\caption{Experimental conditions. WD = Weight Decay}
\label{tab:conditions}
\small
\begin{tabular}{@{}llccl@{}}
\toprule
Condition & Description \\
\midrule
A & WD grokking baseline \\
B & SAM \\
C & SAM + WD\\
D & No regularization (memorizing baseline) \\
\bottomrule
\end{tabular}
\end{table}

For these tasks, we consider $(a + b) \bmod p$ for $p \in \{97, 113\}$. We train 1 and 2-layer transformers, in 4 conditions, summarized in Table~\ref{tab:conditions}. We train each condition on each model size, with 5 seeds (treating each $\rho$ as an independent condition for SAM-based methods). Each $\rho$ signifies a different level of flatness enforced by SAM. Detailed descriptions of the experiments are located in Appendix \ref{sec:exp_details}.

\noindent\textbf{Measuring Acceleration in Generalization} Acceleration is calculated as how much faster methods reach the required accuracy threshold (see Appendix \ref{sec:exp_details}).
Precisely, we define $\mathrm{Accel}_{\rho}$ to be the acceleration of how much faster SAM with weight decay (WD) with $\rho$ generalizes over just WD:

\[\mathrm{Accel}_{\rho} = \operatorname{median}_{s \in \{0,1,2,3,4\}} \mathrm{Accel}_{s,\rho}, \quad \mathrm{Accel}_{s,\rho} = \frac{E^{\mathrm{WD}}_{s}}{E^{\mathrm{SAM+WD}}_{s,\rho}}\]

Here, \(s\) indexes the random seed, \(\rho\) is the SAM radius,
\(E^{\mathrm{WD}}_{s}\) and \(E^{\mathrm{SAM+WD}}_{s,\rho}\) are defined as the first epoch at which the WD or SAM+WD run with seed \(s\)
completes its transition, respectively. The plotted value is the median same-seed acceleration across seeds.

\noindent\textbf{Measuring Curvature}
For modular arithmetic tasks, we verify that SAM approaches flatter solutions through 3 curvature metrics. First, we estimate the largest training-loss Hessian eigenvalue using Hessian-vector products computed by automatic differentiation, avoiding explicit Hessian materialization, and apply 50 iterations of the Lanczos method to the empirical training-loss Hessian with respect to all trainable parameters~\citep{pearlmutter1994fast,lanczos1950iteration,yao2020pyhessian}. We additionally report the scale-normalized variant \(\lambda_{\max}/\|\theta\|_2^2\) of this. Lastly, we also report the relative flatness metric as implemented by \citet{han2025flatness} who studied flatness in grokking, although we believe the Hessian eigenvalue metrics provide a more accurate estimation of flatness (See Appendix \ref{sec:han_bad}). The full tables containing all results are located in Appendix \ref{tab:nanda_full_1}, \ref{tab:nanda_full_2}. 

\paragraph{Flatness does not Drive Generalization}

\begin{figure}[t]
\centering
\includegraphics[width=0.7\textwidth]{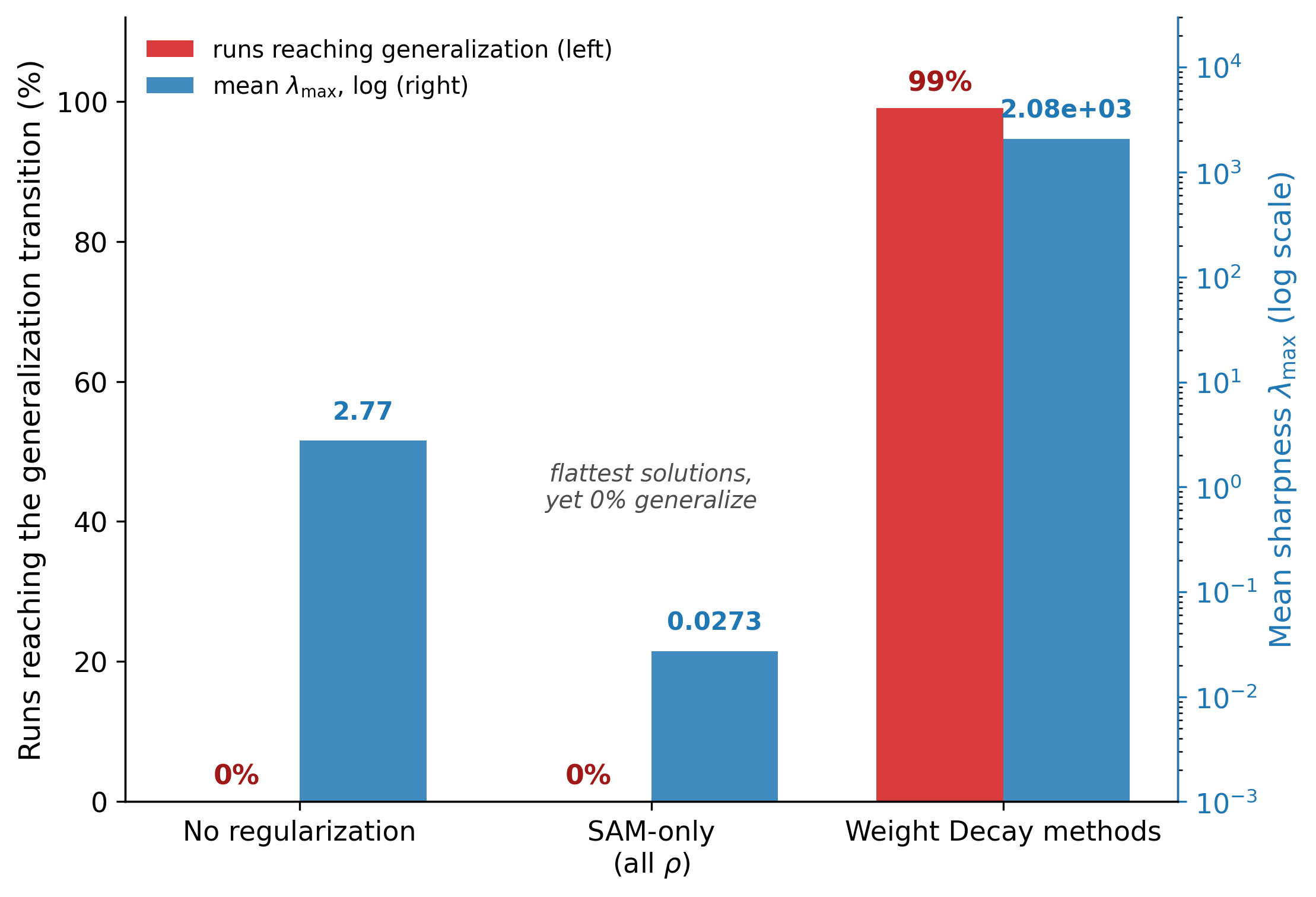}
\caption{
\textbf{Flatness alone does not explain the generalization transition.} Red bars (left axis)
show the percentage of modular arithmetic runs reaching the generalization
transition; blue bars (right axis, log scale) show the mean final-checkpoint
sharpness $\lambda_{\max}$ of the training-loss Hessian. SAM-only finds by far the flattest solutions yet no run generalizes, while
weight-decay-based runs generalize in 99\% of cases despite substantially
sharper final solutions. No regularization pools 20 runs (4 settings
$\times$ 5 seeds); SAM-only pools 180 runs across several values of $\rho$;
Weight Decay methods pool WD and SAM+WD runs across all values of $\rho$ (200 runs).
}
\label{fig:flatness_bars}
\end{figure}

Figure~\ref{fig:flatness_bars} shows our results training all experimental conditions on modular arithmetic tasks. The results demonstrate that SAM-only finds flattest solutions, evidenced by the lowest maximum eigenvalues of the Hessian. Tables~\ref{tab:nanda_full_1}, ~\ref{tab:nanda_full_2} shows that this is true for other flatness metrics as well. Despite being flatter solutions, however, these models do not grok in any setting, regardless of the value of $\rho$ or seed.  This suggests that flatness of the underlying solution by itself is not sufficient to produce the generalization transition. Conversely, weight decay methods often find sharper solutions, and yet exhibit grokking behavior on nearly all settings, achieving high validation accuracy. Thus, our results highlight a lack of any direct link between flatness of the minima and the presence of grokking. These results corroborate our theoretical insights that flatness alone cannot act as a reliable selector between the memorizing and generalizing branches while weight norm can.

\paragraph{Flat Solutions Transition Faster}

\begin{figure*}[t]
\centering
\includegraphics[width=\textwidth]{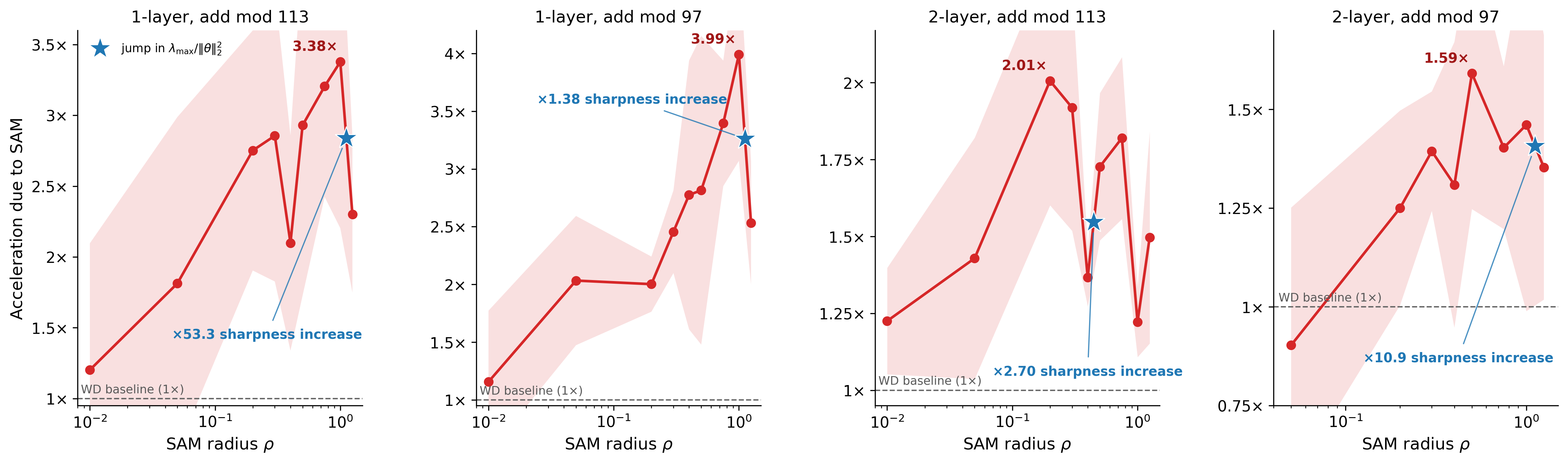}
\caption{
\textbf{SAM accelerates generalization.} Acceleration of the grokking transition under SAM+WD relative to WD alone, as a
function of the SAM radius $\rho$, across Transformer architectures on modular
addition. Acceleration is the median per-seed ratio of the WD grokking epoch to
the SAM+WD grokking epoch; the dashed line marks the WD baseline ($1\times$).
Shaded regions show $\pm 1$ standard deviation of the SAM+WD grokking epoch
across 5 seeds. Blue stars mark increases in
the scale-normalized sharpness $\lambda_{\max}/\|\theta\|_2^2$, indicating where first-order Taylor estimates from SAM begin to degrade and can no longer enforce flatness.
}
\label{fig:mod_arithmetic_accel}
\end{figure*}

Figure~\ref{fig:mod_arithmetic_accel} shows that 
pairing SAM with weight decay substantially accelerates the transition to generalizing solutions across all modular arithmetic tasks. In the addition task with $p=97$ and a 1-layer transformer, for instance, the appropriate value of $\rho$ can achieve up to $\approx4$x epoch-level acceleration compared to using only weight decay. We observe that the acceleration tends to increase with the size of $\rho$, suggesting a relationship between flatness regularization strength and the time it takes to make the generalizing transition. This effect, however, is not indefinite. As $\rho$ grows over a certain threshold, acceleration can start to slow down rapidly. This slowdown may be attributed to the first-order approximation used by SAM to estimate the maximum loss in a local neighborhood, which becomes inaccurate as $\rho$ increases beyond a certain threshold. Appendix~\ref{sec:full_table} supports this, showing that the maximum eigenvalue of the Hessian decreases as $\rho$ increases up to the same approximate threshold that shows a decrease in generalization speed in Figure~\ref{fig:mod_arithmetic_accel}. These results suggest that SAM can no longer meaningfully increase flatness as rho gets too large, which can cause an eventual decline in grokking speed. Together, these results show that moving towards flatter solutions with lower weight norm provides good guidance to the generalizing trajectory. 

\begin{figure*}[t]
\centering

\begin{subfigure}{0.49\textwidth}
    \centering
    \includegraphics[width=\linewidth,height=0.15\textheight]{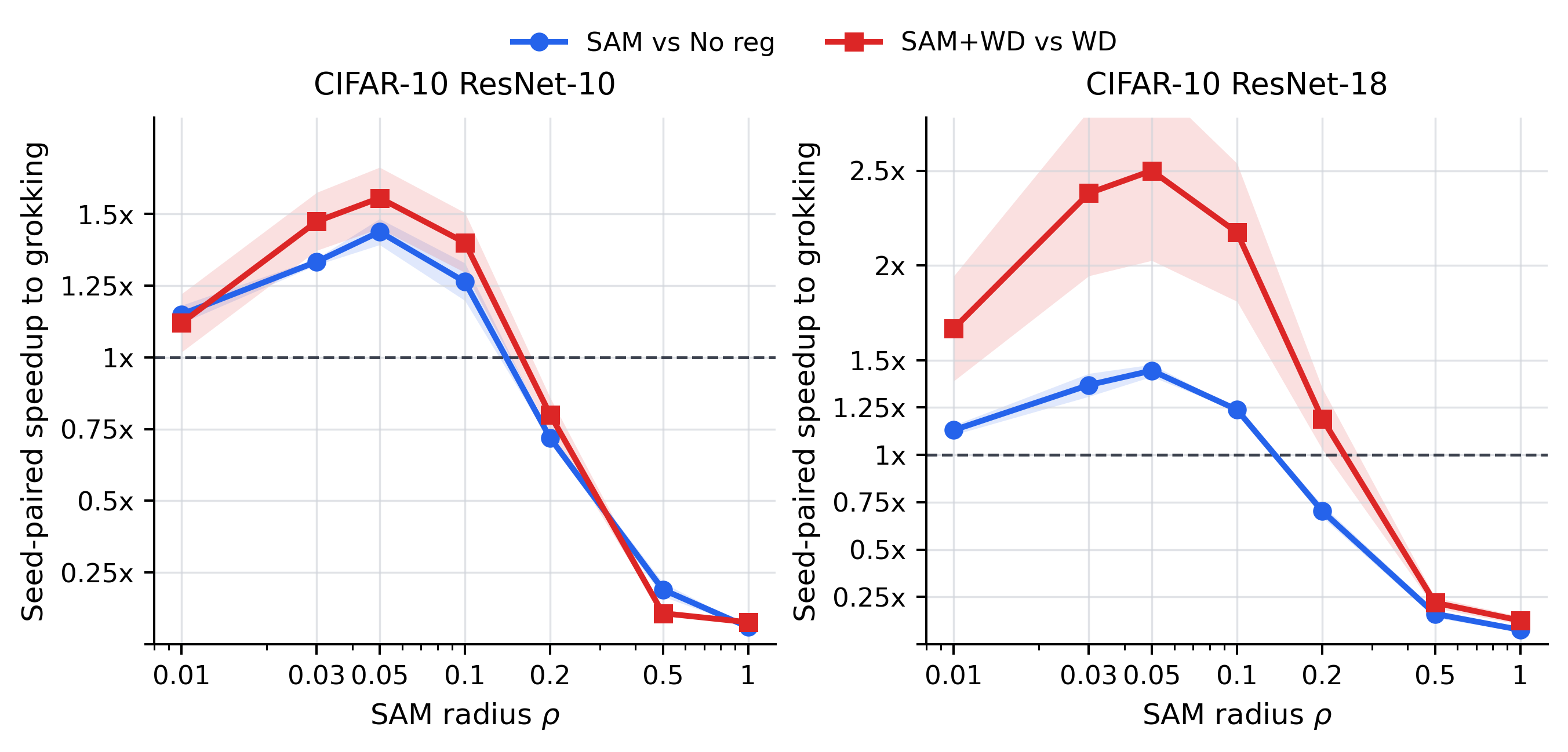}
    \label{fig:cifar_accel}
\end{subfigure}
\hfill
\begin{subfigure}{0.49\textwidth}
    \centering
    \includegraphics[width=\linewidth,height=0.15\textheight]{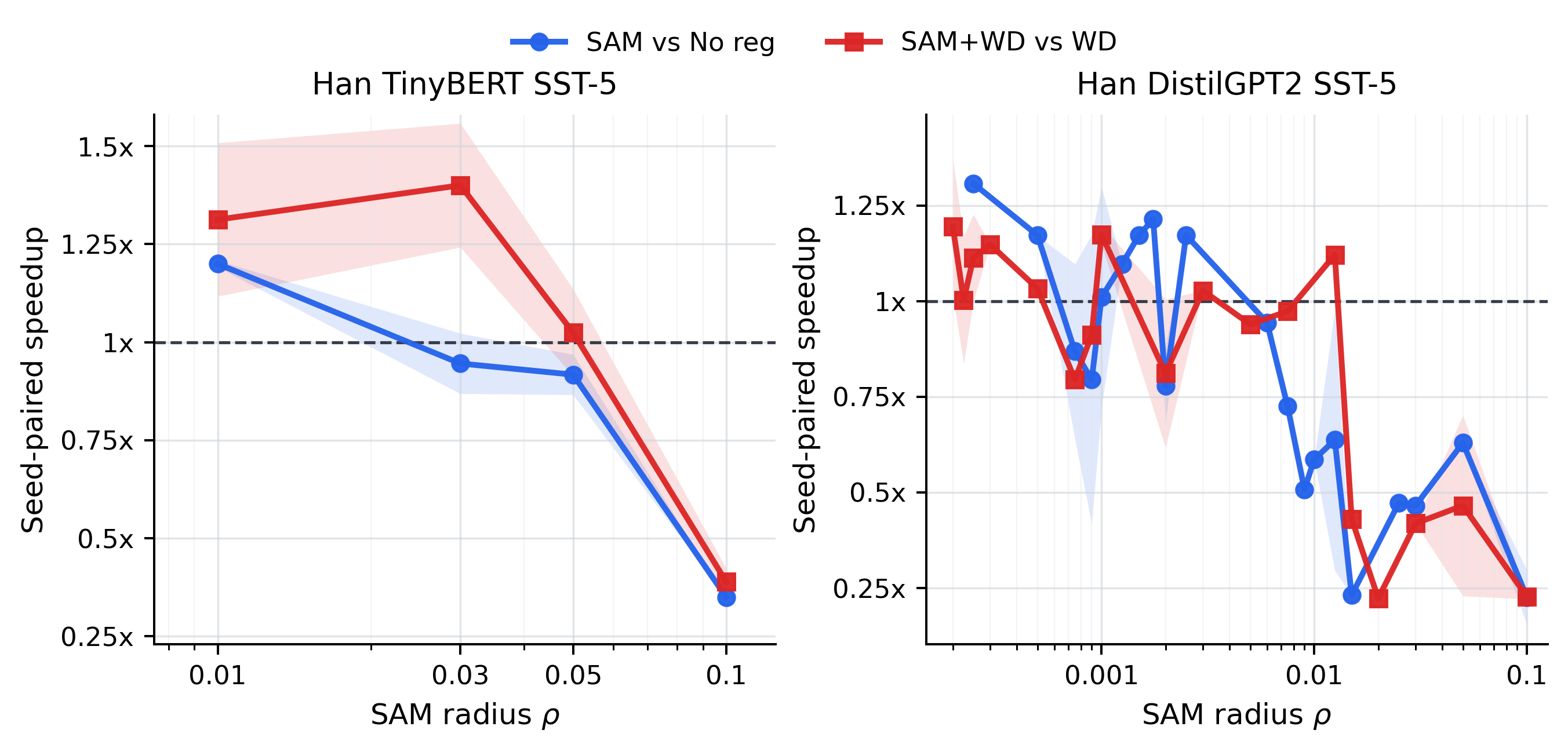}
    \label{fig:language_accel}
\end{subfigure}

\caption{
Relationship between SAM radius $\rho$ and grokking speedup across models trained on CIFAR-10 (left) and SST-5 (right).
Larger $\rho$ generally induces stronger flatness regularization and leads to a faster grokking transition when combined with weight decay until $\rho$ is too large. Shaded regions indicate one standard deviation above and below the median across 3 seeds.
}
\label{fig:han_accel}
\end{figure*}

\begin{figure}[t]
\centering
 \includegraphics[width=0.75\textwidth]{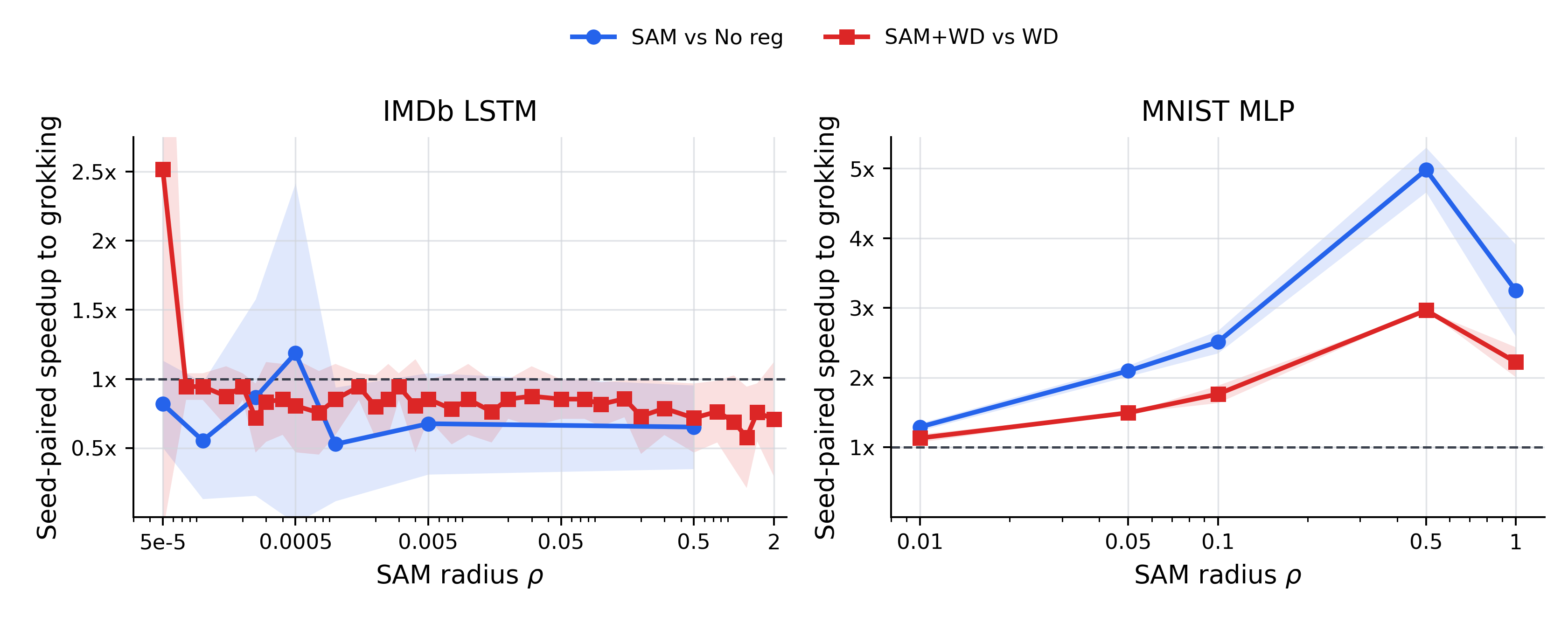}
\caption{Relationship between SAM radius $\rho$ and grokking speedup across models trained on MNIST and IMDB as defined by Omnigrok.
Larger $\rho$ generally induces stronger flatness regularization and leads to faster grokking transition when combined with weight decay until $\rho$ is too large. Shaded regions indicate one standard deviation above and below the median across 3 seeds.}
\label{fig:omnigrok_accel}
\end{figure}

\subsection{Accelerating Generalization on Practical Tasks}

These results indicate broader applicability than on modular arithmetic. Several datasets across vision and language tasks are observed to exhibit grokking-like phenomena. \citet{liu2023omnigrok} demonstrate delayed generalization in grokking through a mismatch between the training and test image, language, and molecular datasets, known as Omnigrok. \citet{han2025flatness} demonstrate delayed grokking through regularization for sharpness across vision and language tasks. We recreate these setups with the addition of SAM, showing that, while difficult to isolate core mechanisms of generalization in these less interpretable datasets, SAM continues the pattern of accelerating this transition when these core mechanisms are already present. Figure~\ref{fig:han_accel} demonstrates the relationship between flatness and the speed of generalization by training ResNets on CIFAR-10 and TinyBERT and DistilGPT2 on SST-5. Figure~\ref{fig:omnigrok_accel} demonstrates this trend on the setup on LSTMs and MLPs used in Omnigrok. Full experimental details are provided in the Appendix.

\subsection{Results}

Our experimental results highlight three key findings across all tasks, including modular arithmetic, image classification, and natural language processing:

\begin{enumerate}
    \item SAM-based training results in solutions that are the flattest, but those solutions fail to generalize in interpretable tasks where other factors are known to induce grokking.
    \item SAM-based training, combined with conditions that reliably induce generalization (e.g. weight decay), makes the transition to generalizing solutions far quicker than with weight decay alone.
    \item The acceleration of the transition to generalization regimes is well explained by  the increase in perturbation radius $\rho$ up to the threshold where SAM-based methods can meaningfully induce flatness.
\end{enumerate}

These results show that seeking flatter minima by itself does not make the transition to the generalizing branch, but accelerates the generalization transition on tasks where grokking already happens (usually through weight decay), corroborating our hypothesis. 

\section{Limitations}
\label{sec:limits}

In this section, we discuss limitations of our work. First, our theoretical analysis that isolates the roles of flatness, norm, and local stability applies to an analytically tractable setting, and does not fully model the extensive architectures used in our experiments, such as Transformers, ResNets, LSTMs, etc. The theoretical analysis also makes assumptions on the loss and the quadratic approximation to SAM that make the separation of memorization and generalization phases transparent, but do not imply that the same stability thresholds or norm comparisons hold unchanged for cross-entropy loss, attention layers, normalization layers, or stochastic minibatch training. Our conclusions are also scoped to the optimization regimes we study. We demonstrate that flatness cannot cause generalization alone through modular arithmetic, a more interpretable task where weight decay has been shown to be key to grokking. In tasks on image classification and natural language processing, we do not show this, because isolating the impact of weight decay and other factors that may cause generalization is significantly more challenging. We do, however, demonstrate that the relationship between the speed of generalization and the strength of flatness regularization through SAM persists on all tasks. Lastly, SAM introduces additional computational costs that may limit its use in practice. Each SAM update requires computing an adversarial perturbation in parameter space, which adds an extra forward-backward pass compared with standard training. This overhead may offset some of the benefits in speed from reaching the generalizing regime in fewer epochs. Our speedup metric focuses on epoch-level transition time rather than total compute, memory cost, or energy usage, as the focus of this paper is on interpretability and understanding the role of flatness, rather than producing new optimizers or regularization. Future work may consider variations of SAM that retain many of the same benefits with lower computational cost \citep{becker2025momentum, ji2024single}.

\section{Conclusion}
\label{sec:conclusion}

In this work, we study whether flatness can act as a causal mechanism for driving generalization in regimes where memorizing solutions are more favored. We revisit the role of flatness in grokking and find that it is not necessary for generalization, challenging existing work. 
Through sharpness-aware minimization, we disentangle the influence of flatness and find that while it cannot reliably drive generalization, it can accelerate the transition significantly. We provide new evidence for a strong relationship between the approximate strength of flatness regularization and the speed of generalization, uncovering new perspectives on the role of flatness that extend meaningfully to practical settings, including computer vision and natural language processing tasks across several architectures.

We provide a theoretical construction that can explain this behavior through an algorithmic separation between flatness, norm, and stability. In a two-layer ReLU model with both memorizing and generalizing interpolants, trace flatness is identical and cannot distinguish them. Parameter norm, however, separates the branches: the generalizing interpolant uses shared structure and has lower norm than the memorizing interpolant once the training set is sufficiently large. This identifies how weight decay can act as the mechanism that selects generalizing structure, while SAM destabilizes memorizing solutions in the low-norm regime, thereby accelerating escape, while being unable to move optimization toward the generalizing branch by itself.
Together, these results better reconcile two threads of research: flatness and grokking, while casting  further light onto causal mechanisms for generalization.

\section{Acknowledgments}
This work was done at Purdue University. We thank Professor Rajiv Khanna and Young In Kim at Purdue University for their valuable insights, especially towards developing theoretical models to explain these results.

\bibliographystyle{exactrefs}
\bibliography{references}
\nocite{*}

\clearpage
\appendix
\raggedbottom

\section{Appendix}

\subsection{A note on Relative Flatness~\citep{han2025flatness}}
\label{sec:han_bad}
In addition to measuring sharpness using methods based on the maximum eigenvalue of the Hessian during training, we report sharpness on the relative flatness metrics used by prior works~\cite{han2025flatness}.

Let \(p_\theta(x)\) denote the softmax probability vector, \(\phi_\theta(x)\) the penultimate representation used by the final classifier, and \(W\) the final unembedding/classifier matrix. The base final-layer Hessian factor is
\[
\mathrm{HanH}=\frac{1}{n}\sum_{i=1}^n
\left(1-\|p_\theta(x_i)\|_2^2\right)\|\phi_\theta(x_i)\|_2^2 .
\]
Relative flatness is defined as
\[
\mathrm{HanRF}=\|W\|_F^2\,\mathrm{HanH}.
\]

While this metric may not be as consistent as the other flatness metrics we use, our results do not necessarily contradict \citet{han2025flatness}, due to this metric not being invariant to several function-preserving changes in the classifier. For example, this quantity can change substantially under changes in confidence or scaling even though the final induced classifier is unchanged. Thus, an observed increase in this metric along a grokking trajectory does not by itself imply that flatness is the causal variable selecting the generalizing solution.

\subsection{Experimental Details}
\label{sec:exp_details}

\paragraph{Modular Arithmetic} For modular arithmetic tasks, we consider $(a + b) \bmod p$ for $p \in \{97, 113\}$. For each setting, all $p(p-1)/2$ unordered pairs are split 30\%/70\% into train and validation sets with a fixed seed. These pairs are generated synthetically through uniform sampling. Input tokens are $[a, b, \texttt{=}]$; the target is the residue. We employ a single-layer Transformer ($d_{\text{model}}{=}128$, 4 heads, $d_{\text{mlp}}{=}512$, $\sim$227K parameters), and a 2-layer Transformer with the same width ($\sim$427K parameters). All experiments use full-batch training, learning rate $10^{-3}$, and AdamW~\citep{loshchilov2019decoupled}, and cross-entropy loss. This experimental setup is an adaptation of existing work by \cite{nanda2023progress}.

For SAM-based methods (SAM-only and SAM with weight decay), we use the following values of $\rho$: \{0.01, 0.05, 0.2, 0.3, 0.4, 0.5, 0.75, 1.0, 1.25\}. We consider a model to have fully transitioned once it has reached a generalizing interpolating solution with both training and validation accuracy above 99\%.

\paragraph{Relative Flatness in Grokking}
\citep{han2025flatness} demonstrates delayed grokking through regularization for sharpness. We recreate several of their tasks to demonstrate specific ways flatness can be used to actually accelerate their grokking on image and language tasks.

We train ResNet-10s and ResNet-18s on CIFAR-10 and fine-tune DistilGPT2 and TinyBERT on SST-5 as outlined in their work. The threshold to determine when a generalizing interpolant is reached is 99\% training set accuracy, 75\% validation set accuracy for CIFAR-10, and 99\% training set accuracy and 45\% validation set accuracy for language tasks on SST-5.

\paragraph{Omnigrok} Omnigrok demonstrates delayed generalization in grokking through a mismatch between the training and test image, language, and molecular datasets. We recreate their setup and demonstrate the flatness-generalization relationship on their image and language tasks.

Specifically, we train MLPs on MNIST and LSTMs on IMDB as they outline and pair them with SAM. The respective thresholds to determine if a generalizing interpolant is reached are 60\% and 54\%, respectively.

\subsection{Proofs}
\label{app:proofs}

\begin{proof}[Proof of Theorem~\ref{thm:main_separation}]
We prove the three claims in order:

\noindent\textbf{1. Flatness does not separate.}
For the memorizing interpolant, recall that for each training point $x_i$, $w_i=r x_i, b_i=-r(d-1), a_i=\frac{y_i}{r}.$ At $x_i$, 
$
w_i^\top x_i+b_i
=
r x_i^\top x_i-r(d-1)
=
rd-r(d-1)
=
r>0.
$
For any $x\neq x_i$ in $\{-1,+1\}^d$, we have $x_i^\top x\le d-2$, and hence
$
w_i^\top x+b_i
\le r(d-2)-r(d-1)
=
-r<0.
$
Thus neuron $i$ activates only on $x_i$, and
$
f_{\theta_{\rm mem}}(x_i)=a_i r=y_i.
$
So the memorizing construction interpolates the training set.

For the $(C,r)$-generalizing interpolant, 
$
z=(x_1,\ldots,x_C).
$
The neuron indexed by this $z$ satisfies
$
w_z^\top x+b_z
=
r\sum_{k=1}^C z_kx_k-r(C-1)
=
rC-r(C-1)
=
r>0.
$
Every other pattern $z'\neq z$ differs in at least one coordinate, so
$
\sum_{k=1}^C z'_kx_k\le C-2,
$
and hence
$
w_{z'}^\top x+b_{z'}
\le r(C-2)-r(C-1)
=
-r<0.
$
Therefore exactly one neuron activates, and
$
f_{\theta_{C,r}}(x)
=
a_z r
=
\frac{z_1z_2}{r}r
=
z_1z_2
=
x_1x_2.
$
Thus $\theta_{C,r}$ computes the target rule on the full hypercube.

At an interpolating differentiable solution for squared loss,
\[
H(\theta)
=
\frac{1}{n}\sum_{i=1}^n g_i g_i^\top,
\qquad
g_i:=\nabla_\theta f_\theta(x_i).
\]
Therefore
\[
\operatorname{Tr}H(\theta)
=
\frac{1}{n}\sum_{i=1}^n \|g_i\|_2^2.
\]
For both constructions, exactly one neuron is active on $x_i$. For the active neuron,
$
\frac{\partial f}{\partial a}=r,
\frac{\partial f}{\partial w}=a x_i=\frac{y_i}{r}x_i,
\frac{\partial f}{\partial b}=a=\frac{y_i}{r}.
$
Thus
\[
\|g_i\|_2^2
=
r^2+\frac{\|x_i\|_2^2}{r^2}+\frac{1}{r^2}
=
r^2+\frac{d+1}{r^2}
=
q_d(r).
\]
Averaging over $i$ gives
\[
\operatorname{Tr}H(\theta_{\rm mem}(r))
=
\operatorname{Tr}H(\theta_{C,r})
=
q_d(r).
\]

\noindent\textbf{2. Norm separates.}
Define
$
A_d:=d+(d-1)^2,
A_C:=C+(C-1)^2.
$
For the memorizing interpolant, each of the $n$ neurons satisfies
$
\|w_i\|_2^2=r^2d,
b_i^2=r^2(d-1)^2,
a_i^2=r^{-2}.
$
Hence
\[
\|\theta_{\rm mem}(r)\|_2^2
=
n(A_d r^2+r^{-2}).
\]
For the $(C,r)$-generalizing interpolant, there are $2^C$ neurons, each satisfying
$
\|w_z\|_2^2=r^2C,
b_z^2=r^2(C-1)^2,
a_z^2=r^{-2}.
$
Thus,
$
\|\theta_{C,r}\|_2^2
=
2^C(A_C r^2+r^{-2}).
$
Therefore if $
n>
\frac{2^C(A_C r^2+r^{-2})}{A_d r^2+r^{-2}}.
$,
\[
\|\theta_{C,r}\|_2^2 < \|\theta_{\rm mem}(r)\|_2^2.
\]

\noindent\textbf{3. SAM destabilizes memorization in a low-norm stability window.}
For the memorizing interpolant, the gradients $g_i$ have disjoint parameter support, so they are mutually orthogonal. Since
$
\|g_i\|_2^2=q_d(r),
$
the nonzero eigenvalues of
$
H(\theta_{\rm mem}(r))
=
\frac{1}{n}\sum_{i=1}^n g_i g_i^\top
$
are all equal to
$
\ell_{\rm mem}(r)
:=
\frac{q_d(r)}{n}.
$
Thus
$
\lambda_{\max}(H(\theta_{\rm mem}(r)))=\ell_{\rm mem}(r).
$

Under GD, the `top-mode' multiplier is
$
m_{\rm GD}(r)=1-\eta \ell_{\rm mem}(r).
$
Therefore GD is locally stable iff
$
|m_{\rm GD}(r)|<1,
$
equivalently
$
\ell_{\rm mem}(r)<\ell_{\rm GD}:=\frac{2}{\eta}.
$

Under the quadratic SAM approximation,
$
H_{\rm SAM}=H(I+\gamma H),
\gamma=\rho/\alpha>0.
$
Thus the effective top curvature is
$
\ell_{\rm mem}(r)(1+\gamma \ell_{\rm mem}(r)),
$
and the top-mode SAM multiplier is
$
m_{\rm SAM}(r)
=
1-\eta \ell_{\rm mem}(r)(1+\gamma \ell_{\rm mem}(r)).
$
SAM is locally unstable whenever
$
|m_{\rm SAM}(r)|>1,
$
which holds when
$
\eta \ell_{\rm mem}(r)(1+\gamma \ell_{\rm mem}(r))>2.
$
The SAM stability threshold is the positive solution of
$
\eta \ell(1+\gamma \ell)=2,
$
namely
$
\ell_{\rm SAM}
=
\frac{\sqrt{1+8\gamma/\eta}-1}{2\gamma}.
$
Since $\gamma>0$,
$
0<\ell_{\rm SAM}<\ell_{\rm GD}.
$

Therefore, whenever
$
\ell_{\rm SAM}
<
\ell_{\rm mem}(r)
<
\ell_{\rm GD},
$
we have
$
|m_{\rm GD}(r)|<1
\text{ and }
|m_{\rm SAM}(r)|>1.
$
Thus the memorizing interpolant is locally stable under GD but locally unstable under SAM. Perturbations in the top Hessian eigendirection decay geometrically under GD and grow geometrically under SAM.

Equivalently, this condition defines a nonempty low-norm stability window whenever the equation
$
\frac{q_d(r)}{n}\in(\ell_{\rm SAM},\ell_{\rm GD})
$
has solutions. On the low-\(r\) branch, this corresponds to an interval
$
r\in(r_{\rm low},r_{\rm high}).
$
\end{proof}

\subsection{Full Tables, Grokking in Modular Arithmetic}
\label{sec:full_table}

We provide a table with all modular arithmetic results and three sharpness metrics, including a parameter-normalized variant to the maximum eigenvalue of the Hessian, and the sharpness metrics discussed in other works \cite{han2025flatness}.

\begin{table}[H]
\centering
\scriptsize
\setlength{\tabcolsep}{2.5pt}
\renewcommand{\arraystretch}{0.93}
\textbf{(a) 1-layer Transformer, add mod 113}\par\vspace{0.25em}
\resizebox{\linewidth}{!}{%
\begin{tabular}{llccccccc}
\toprule
Condition & $\rho$ & Grok & Grok epoch & Val. acc. & Speedup & Han RF $\downarrow$ & $\lambda_{\max}^{50}\downarrow$ & $\lambda_{\max}^{50}/\|\theta\|_2^2\downarrow$ \\
\midrule
No reg & -- & 0/5 & -- & 0.307 $\pm$ 0.0242 & -- & 48.02 $\pm$ 37.47 & 0.00213 $\pm$ 0.00157 & 4.54e-07 $\pm$ 3.37e-07 \\
WD & -- & 5/5 & 5538 $\pm$ 2976.6 & 1 $\pm$ 4.06e-04 & 1x & 9.247 $\pm$ 10.17 & 0.0547 $\pm$ 0.0553 & 1.25e-05 $\pm$ 7.95e-06 \\
SAM-only & 0.01 & 0/5 & -- & 0.307 $\pm$ 0.0144 & -- & 8.757 $\pm$ 3.058 & 6.67e-04 $\pm$ 2.51e-04 & 1.56e-07 $\pm$ 6.39e-08 \\
SAM-only & 0.05 & 0/5 & -- & 0.351 $\pm$ 0.0711 & -- & 1.026 $\pm$ 0.483 & 1.58e-04 $\pm$ 9.46e-05 & 4.08e-08 $\pm$ 2.42e-08 \\
SAM-only & 0.2 & 0/5 & -- & 0.433 $\pm$ 0.144 & -- & 5.152 $\pm$ 6.666 & 8.43e-04 $\pm$ 0.00126 & 2.39e-07 $\pm$ 3.61e-07 \\
SAM-only & 0.3 & 0/5 & -- & 0.413 $\pm$ 0.172 & -- & 10.99 $\pm$ 15.36 & 0.00177 $\pm$ 0.00263 & 5.02e-07 $\pm$ 7.56e-07 \\
SAM-only & 0.4 & 0/5 & -- & 0.396 $\pm$ 0.0949 & -- & 0.936 $\pm$ 0.984 & 1.37e-04 $\pm$ 1.68e-04 & 3.91e-08 $\pm$ 4.92e-08 \\
SAM-only & 0.5 & 0/5 & -- & 0.424 $\pm$ 0.13 & -- & 0.275 $\pm$ 0.156 & 2.95e-05 $\pm$ 7.79e-06 & 7.76e-09 $\pm$ 1.73e-09 \\
SAM-only & 0.75 & 0/5 & -- & 0.449 $\pm$ 0.142 & -- & 2.656 $\pm$ 2.788 & 2.46e-04 $\pm$ 2.54e-04 & 5.62e-08 $\pm$ 5.98e-08 \\
SAM-only & 1.25 & 0/5 & -- & 0.421 $\pm$ 0.0948 & -- & 0.109 $\pm$ 0.0897 & 1.26e-05 $\pm$ 4.99e-06 & 2.25e-09 $\pm$ 7.91e-10 \\
SAM+WD & 0.01 & 5/5 & 4607.3 $\pm$ 3427.5 & 0.999 $\pm$ 8.28e-04 & 1.202x & 1754.0 $\pm$ 3037.0 & 3.667 $\pm$ 6.348 & 0.00144 $\pm$ 0.00249 \\
SAM+WD & 0.05 & 5/5 & 3054.6 $\pm$ 1979.6 & 1 $\pm$ 2.00e-04 & 1.813x & 45.89 $\pm$ 51.04 & 0.115 $\pm$ 0.143 & 4.06e-05 $\pm$ 4.89e-05 \\
SAM+WD & 0.2 & 5/5 & 2012.4 $\pm$ 620.5 & 1 $\pm$ 0 & 2.752x & 430.9 $\pm$ 718.1 & 0.403 $\pm$ 0.635 & 3.71e-04 $\pm$ 6.15e-04 \\
SAM+WD & 0.3 & 5/5 & 1939.1 $\pm$ 698.6 & 1 $\pm$ 1.00e-04 & 2.856x & 1729.5 $\pm$ 1634.5 & 1.787 $\pm$ 1.111 & 0.00152 $\pm$ 0.00147 \\
SAM+WD & 0.4 & 5/5 & 2638.4 $\pm$ 956.3 & 1 $\pm$ 9.36e-05 & 2.099x & 1188.8 $\pm$ 891.5 & 1.055 $\pm$ 0.637 & 9.22e-04 $\pm$ 6.98e-04 \\
SAM+WD & 0.5 & 5/5 & 1890.1 $\pm$ 775.9 & 1 $\pm$ 3.88e-04 & 2.93x & 1561.5 $\pm$ 959.4 & 1.336 $\pm$ 0.54 & 0.00105 $\pm$ 4.99e-04 \\
SAM+WD & 0.75 & 5/5 & 1727.9 $\pm$ 421.2 & 1 $\pm$ 2.79e-04 & 3.205x & 1599.1 $\pm$ 634.5 & 1.234 $\pm$ 0.449 & 0.00116 $\pm$ 5.46e-04 \\
SAM+WD & 1 & 5/5 & 1639.4 $\pm$ 569.6 & 0.999 $\pm$ 0.00126 & 3.378x & 2072.3 $\pm$ 1668.0 & 1.27 $\pm$ 1.104 & 0.00124 $\pm$ 0.00117 \\
SAM+WD & 1.25 & 5/5 & 2406.8 $\pm$ 578.3 & 1 $\pm$ 6.50e-04 & 2.301x & 1.43e+04 $\pm$ 1.28e+04 & 66.05 $\pm$ 91.25 & 0.0661 $\pm$ 0.0886 \\
\bottomrule
\end{tabular}%
}
\vspace{0.75em}
\vspace{0.75em}
\textbf{(b) 1-layer Transformer, add mod 97}\par\vspace{0.25em}
\resizebox{\linewidth}{!}{%
\begin{tabular}{llccccccc}
\toprule
Condition & $\rho$ & Grok & Grok epoch & Val. acc. & Speedup & Han RF $\downarrow$ & $\lambda_{\max}^{50}\downarrow$ & $\lambda_{\max}^{50}/\|\theta\|_2^2\downarrow$ \\
\midrule
No reg & -- & 0/5 & -- & 0.259 $\pm$ 0.0619 & -- & 26.82 $\pm$ 24.49 & 0.00169 $\pm$ 0.00137 & 3.91e-07 $\pm$ 3.15e-07 \\
WD & -- & 5/5 & 11169 $\pm$ 4739.6 & 0.999 $\pm$ 0.00211 & 1x & 15.21 $\pm$ 17.93 & 0.157 $\pm$ 0.209 & 7.04e-05 $\pm$ 1.03e-04 \\
SAM-only & 0.01 & 0/5 & -- & 0.279 $\pm$ 0.0421 & -- & 24.60 $\pm$ 5.632 & 0.00157 $\pm$ 6.41e-04 & 3.89e-07 $\pm$ 1.72e-07 \\
SAM-only & 0.05 & 0/5 & -- & 0.299 $\pm$ 0.0108 & -- & 0.619 $\pm$ 0.228 & 9.26e-05 $\pm$ 2.48e-05 & 2.52e-08 $\pm$ 8.01e-09 \\
SAM-only & 0.2 & 0/5 & -- & 0.303 $\pm$ 0.00707 & -- & 18.16 $\pm$ 14.52 & 0.00195 $\pm$ 0.00152 & 5.21e-07 $\pm$ 3.99e-07 \\
SAM-only & 0.3 & 0/5 & -- & 0.306 $\pm$ 0.0127 & -- & 0.452 $\pm$ 0.478 & 5.40e-05 $\pm$ 4.28e-05 & 1.51e-08 $\pm$ 1.29e-08 \\
SAM-only & 0.4 & 0/5 & -- & 0.309 $\pm$ 0.00878 & -- & 0.125 $\pm$ 0.0732 & 2.20e-05 $\pm$ 7.51e-06 & 6.15e-09 $\pm$ 2.43e-09 \\
SAM-only & 0.5 & 0/5 & -- & 0.308 $\pm$ 0.0063 & -- & 2.347 $\pm$ 2.285 & 2.80e-04 $\pm$ 2.89e-04 & 7.34e-08 $\pm$ 7.84e-08 \\
SAM-only & 0.75 & 0/5 & -- & 0.361 $\pm$ 0.0897 & -- & 0.257 $\pm$ 0.318 & 3.63e-05 $\pm$ 3.39e-05 & 8.71e-09 $\pm$ 8.62e-09 \\
SAM-only & 1.25 & 0/5 & -- & 0.34 $\pm$ 0.0316 & -- & 0.0605 $\pm$ 0.0828 & 9.66e-06 $\pm$ 7.80e-06 & 1.74e-09 $\pm$ 1.37e-09 \\
SAM+WD & 0.01 & 5/5 & 9661.8 $\pm$ 5151.3 & 0.999 $\pm$ 0.00132 & 1.156x & 1935.4 $\pm$ 3348.5 & 6.506 $\pm$ 11.26 & 0.00564 $\pm$ 0.00976 \\
SAM+WD & 0.05 & 5/5 & 5493.9 $\pm$ 1512.5 & 1 $\pm$ 3.96e-04 & 2.033x & 16.76 $\pm$ 24.79 & 0.0498 $\pm$ 0.0741 & 1.46e-05 $\pm$ 1.94e-05 \\
SAM+WD & 0.2 & 5/5 & 5578.9 $\pm$ 666.1 & 1 $\pm$ 3.96e-04 & 2.002x & 670.6 $\pm$ 407.0 & 1.376 $\pm$ 0.318 & 5.93e-04 $\pm$ 2.78e-04 \\
SAM+WD & 0.3 & 5/5 & 4547.6 $\pm$ 663.4 & 0.999 $\pm$ 3.33e-04 & 2.456x & 933.2 $\pm$ 441.5 & 1.38 $\pm$ 0.967 & 9.20e-04 $\pm$ 7.22e-04 \\
SAM+WD & 0.4 & 5/5 & 4024.9 $\pm$ 1689.4 & 1 $\pm$ 4.07e-04 & 2.775x & 1266.5 $\pm$ 609.6 & 1.679 $\pm$ 1.297 & 0.0014 $\pm$ 0.00117 \\
SAM+WD & 0.5 & 5/5 & 3967.7 $\pm$ 1884.3 & 1 $\pm$ 4.60e-04 & 2.815x & 758.9 $\pm$ 310.3 & 1.059 $\pm$ 0.425 & 8.89e-04 $\pm$ 3.05e-04 \\
SAM+WD & 0.75 & 5/5 & 3289.8 $\pm$ 527.1 & 0.998 $\pm$ 0.00286 & 3.395x & 784.1 $\pm$ 327.2 & 0.919 $\pm$ 0.338 & 9.50e-04 $\pm$ 4.24e-04 \\
SAM+WD & 1 & 5/5 & 2797.8 $\pm$ 645.9 & 1 $\pm$ 4.60e-04 & 3.992x & 741.0 $\pm$ 148.0 & 0.731 $\pm$ 0.12 & 7.23e-04 $\pm$ 8.28e-05 \\
SAM+WD & 1.25 & 5/5 & 4411.1 $\pm$ 930.3 & 1 $\pm$ 2.15e-04 & 2.532x & 1119.3 $\pm$ 77.79 & 0.95 $\pm$ 0.0667 & 0.001 $\pm$ 3.11e-05 \\
\bottomrule
\end{tabular}%
}
\vspace{0.75em}
\caption{Modular-addition results for 1-layer Transformers. Grokking and validation summaries use matched five-seed runs. Speedup is the median paired-seed SAM+WD speedup over WD. Han RF is the official Han et al. relative-flatness metric. Hessian columns use the previously computed three-seed Hessian-50 posthoc table matched by task, condition, and $\rho$.}
\label{tab:nanda_full_1}
\end{table}

\begin{table}[H]
\centering
\scriptsize
\setlength{\tabcolsep}{2.5pt}
\renewcommand{\arraystretch}{0.93}
\textbf{(c) 2-layer Transformer, add mod 113}\par\vspace{0.25em}
\resizebox{\linewidth}{!}{%
\begin{tabular}{llccccccc}
\toprule
Condition & $\rho$ & Grok & Grok epoch & Val. acc. & Speedup & Han RF $\downarrow$ & $\lambda_{\max}^{50}\downarrow$ & $\lambda_{\max}^{50}/\|\theta\|_2^2\downarrow$ \\
\midrule
No reg & -- & 0/5 & -- & 0.311 $\pm$ 0.032 & -- & 1.23e+04 $\pm$ 2.12e+04 & 7.04 $\pm$ 12.14 & 0.00107 $\pm$ 0.00184 \\
WD & -- & 5/5 & 11266 $\pm$ 9705.0 & 0.876 $\pm$ 0.174 & 1x & 6.44e+06 $\pm$ 1.99e+06 & 5.91e+04 $\pm$ 2.06e+04 & 83.45 $\pm$ 63.48 \\
SAM-only & 0.01 & 0/5 & -- & 0.41 $\pm$ 0.171 & -- & 3.877 $\pm$ 2.619 & 0.00403 $\pm$ 0.00223 & 7.42e-07 $\pm$ 4.63e-07 \\
SAM-only & 0.05 & 0/5 & -- & 0.44 $\pm$ 0.177 & -- & 900.7 $\pm$ 1558.1 & 0.48 $\pm$ 0.826 & 7.57e-05 $\pm$ 1.30e-04 \\
SAM-only & 0.2 & 0/5 & -- & 0.39 $\pm$ 0.119 & -- & 2.052 $\pm$ 1.191 & 0.00105 $\pm$ 6.43e-04 & 1.75e-07 $\pm$ 1.05e-07 \\
SAM-only & 0.3 & 0/5 & -- & 0.383 $\pm$ 0.101 & -- & 13.66 $\pm$ 20.79 & 0.00301 $\pm$ 0.00365 & 4.87e-07 $\pm$ 5.87e-07 \\
SAM-only & 0.4 & 0/5 & -- & 0.32 $\pm$ 0.029 & -- & 1.164 $\pm$ 1.013 & 4.32e-04 $\pm$ 3.01e-04 & 6.39e-08 $\pm$ 4.44e-08 \\
SAM-only & 0.5 & 0/5 & -- & 0.342 $\pm$ 0.0796 & -- & 3.86 $\pm$ 4.179 & 8.53e-04 $\pm$ 6.96e-04 & 1.17e-07 $\pm$ 9.60e-08 \\
SAM-only & 0.75 & 0/5 & -- & 0.328 $\pm$ 0.0572 & -- & 0.431 $\pm$ 0.586 & 1.48e-04 $\pm$ 1.58e-04 & 1.82e-08 $\pm$ 1.87e-08 \\
SAM-only & 1.25 & 0/5 & -- & 0.301 $\pm$ 0.00727 & -- & 75.43 $\pm$ 130.5 & 0.121 $\pm$ 0.21 & 9.81e-06 $\pm$ 1.70e-05 \\
SAM+WD & 0.01 & 4/5 & 9196.7 $\pm$ 1299.2 & 0.995 $\pm$ 0.0116 & 1.225x & 4.23e+04 $\pm$ 4.08e+04 & 275.3 $\pm$ 342.9 & 0.439 $\pm$ 0.547 \\
SAM+WD & 0.05 & 5/5 & 7883.8 $\pm$ 2164.2 & 1 $\pm$ 2.90e-04 & 1.429x & 5413.5 $\pm$ 4751.9 & 15.15 $\pm$ 9.783 & 0.0267 $\pm$ 0.022 \\
SAM+WD & 0.2 & 5/5 & 5616.2 $\pm$ 1133.3 & 1 $\pm$ 1.23e-04 & 2.006x & 1675.8 $\pm$ 814.8 & 5.077 $\pm$ 0.31 & 0.00841 $\pm$ 8.15e-04 \\
SAM+WD & 0.3 & 5/5 & 5870.8 $\pm$ 1226.2 & 1 $\pm$ 3.68e-04 & 1.919x & 938.9 $\pm$ 172.3 & 2.184 $\pm$ 0.806 & 0.00339 $\pm$ 0.00147 \\
SAM+WD & 0.4 & 5/5 & 8241.4 $\pm$ 569.7 & 1 $\pm$ 4.11e-04 & 1.367x & 1076.5 $\pm$ 33.77 & 2.144 $\pm$ 0.407 & 0.00313 $\pm$ 6.04e-04 \\
SAM+WD & 0.5 & 5/5 & 6523.5 $\pm$ 904.6 & 0.999 $\pm$ 3.79e-04 & 1.727x & 4229.5 $\pm$ 6177.6 & 6.003 $\pm$ 7.931 & 0.00845 $\pm$ 0.0109 \\
SAM+WD & 0.75 & 5/5 & 6190.1 $\pm$ 894.6 & 0.999 $\pm$ 1.87e-04 & 1.82x & 798.0 $\pm$ 104.0 & 1.515 $\pm$ 0.281 & 0.00215 $\pm$ 4.90e-04 \\
SAM+WD & 1 & 5/5 & 9219.3 $\pm$ 860.7 & 1 $\pm$ 3.26e-04 & 1.222x & 4000.0 $\pm$ 3268.2 & 7.145 $\pm$ 6.088 & 0.00972 $\pm$ 0.00934 \\
SAM+WD & 1.25 & 5/5 & 7525.7 $\pm$ 1732.6 & 0.95 $\pm$ 0.11 & 1.497x & 1.21e+04 $\pm$ 8775.4 & 6.843 $\pm$ 3.611 & 0.00787 $\pm$ 0.00576 \\
\bottomrule
\end{tabular}%
}
\vspace{0.75em}
\vspace{0.75em}
\textbf{(d) 2-layer Transformer, add mod 97}\par\vspace{0.25em}
\resizebox{\linewidth}{!}{%
\begin{tabular}{llccccccc}
\toprule
Condition & $\rho$ & Grok & Grok epoch & Val. acc. & Speedup & Han RF $\downarrow$ & $\lambda_{\max}^{50}\downarrow$ & $\lambda_{\max}^{50}/\|\theta\|_2^2\downarrow$ \\
\midrule
No reg & -- & 0/5 & -- & 0.345 $\pm$ 0.0369 & -- & 1669.3 $\pm$ 1582.0 & 4.054 $\pm$ 3.549 & 7.02e-04 $\pm$ 6.27e-04 \\
WD & -- & 4/5 & 6770 $\pm$ 1526.4 & 0.821 $\pm$ 0.133 & 1x & 3.84e+06 $\pm$ 3.58e+06 & 2.33e+04 $\pm$ 3.02e+04 & 20.06 $\pm$ 21.73 \\
SAM-only & 0.01 & 0/5 & -- & 0.311 $\pm$ 0.0102 & -- & 213.0 $\pm$ 304.3 & 0.339 $\pm$ 0.402 & 6.14e-05 $\pm$ 7.38e-05 \\
SAM-only & 0.05 & 0/5 & -- & 0.329 $\pm$ 0.0212 & -- & 6.473 $\pm$ 4.836 & 0.00978 $\pm$ 0.00722 & 2.01e-06 $\pm$ 1.67e-06 \\
SAM-only & 0.2 & 0/5 & -- & 0.324 $\pm$ 0.0394 & -- & 5.838 $\pm$ 6.646 & 0.00533 $\pm$ 0.00553 & 9.62e-07 $\pm$ 1.04e-06 \\
SAM-only & 0.3 & 0/5 & -- & 0.303 $\pm$ 0.0106 & -- & 0.402 $\pm$ 0.396 & 1.92e-04 $\pm$ 1.51e-04 & 3.16e-08 $\pm$ 2.38e-08 \\
SAM-only & 0.4 & 0/5 & -- & 0.302 $\pm$ 0.0096 & -- & 0.3 $\pm$ 0.384 & 2.72e-04 $\pm$ 2.49e-04 & 4.28e-08 $\pm$ 3.83e-08 \\
SAM-only & 0.5 & 0/5 & -- & 0.298 $\pm$ 0.00274 & -- & 0.885 $\pm$ 1.209 & 0.00102 $\pm$ 0.00152 & 1.51e-07 $\pm$ 2.29e-07 \\
SAM-only & 0.75 & 0/5 & -- & 0.296 $\pm$ 0.0018 & -- & 0.309 $\pm$ 0.43 & 1.25e-04 $\pm$ 1.59e-04 & 1.58e-08 $\pm$ 1.98e-08 \\
SAM-only & 1.25 & 0/5 & -- & 0.297 $\pm$ 0.00145 & -- & 1.696 $\pm$ 2.799 & 8.37e-04 $\pm$ 0.00122 & 7.40e-08 $\pm$ 1.07e-07 \\
SAM+WD & 0.01 & 5/5 & 5572.0 $\pm$ 1413.1 & 0.993 $\pm$ 0.0158 & 1.215x & 1.33e+04 $\pm$ 6963.0 & 189.2 $\pm$ 211.1 & 0.166 $\pm$ 0.12 \\
SAM+WD & 0.05 & 5/5 & 7497.2 $\pm$ 2891.2 & 0.999 $\pm$ 1.98e-04 & 0.903x & 5149.1 $\pm$ 3155.8 & 15.29 $\pm$ 7.406 & 0.0286 $\pm$ 0.0154 \\
SAM+WD & 0.2 & 5/5 & 5416.0 $\pm$ 1068.8 & 1 $\pm$ 3.04e-04 & 1.25x & 3712.1 $\pm$ 4173.1 & 14.14 $\pm$ 17.51 & 0.0236 $\pm$ 0.0287 \\
SAM+WD & 0.3 & 5/5 & 4856.5 $\pm$ 526.9 & 1 $\pm$ 3.72e-04 & 1.394x & 1070.3 $\pm$ 352.8 & 2.694 $\pm$ 0.732 & 0.00421 $\pm$ 0.00107 \\
SAM+WD & 0.4 & 5/5 & 5171.9 $\pm$ 1428.9 & 1 $\pm$ 1.73e-04 & 1.309x & 2177.1 $\pm$ 2167.2 & 7.391 $\pm$ 7.756 & 0.0109 $\pm$ 0.0108 \\
SAM+WD & 0.5 & 5/5 & 4255.2 $\pm$ 918.8 & 1 $\pm$ 2.63e-04 & 1.591x & 2191.7 $\pm$ 1788.9 & 5.266 $\pm$ 3.768 & 0.00745 $\pm$ 0.00445 \\
SAM+WD & 0.75 & 5/5 & 4825.4 $\pm$ 708.8 & 0.999 $\pm$ 3.62e-04 & 1.403x & 5992.0 $\pm$ 7701.6 & 11.68 $\pm$ 15.58 & 0.02 $\pm$ 0.0282 \\
SAM+WD & 1 & 5/5 & 4633.8 $\pm$ 1497.1 & 1 $\pm$ 4.75e-04 & 1.461x & 1294.1 $\pm$ 838.0 & 2.02 $\pm$ 1.294 & 0.00265 $\pm$ 0.00155 \\
SAM+WD & 1.25 & 5/5 & 5003.7 $\pm$ 1238.2 & 1 $\pm$ 4.07e-04 & 1.353x & 8109.4 $\pm$ 569.8 & 19.08 $\pm$ 3.056 & 0.029 $\pm$ 0.00289 \\
\bottomrule
\end{tabular}%
}
\caption{Modular-addition results for 2-layer Transformers. Grokking and validation summaries use matched five-seed runs. Speedup is the median paired-seed SAM+WD speedup over WD. Han RF is the official Han et al. relative-flatness metric. Hessian columns use the previously computed three-seed Hessian-50 posthoc table matched by task, condition, and $\rho$.}
\label{tab:nanda_full_2}
\end{table}

\subsection{Computational Costs}
\label{app:compute}

We report computational cost in GPU-hours, computed as the sum of single-GPU wall-clock time across runs. All training runs used a single GPU; we do not normalize for GPU model, since experiments were scheduled opportunistically across A10, A30, A100 devices. SAM-based methods, including SAM and SAM+WD, require two gradient evaluations per optimizer step, requiring more time. However, the required time varies largely across conditions and tasks, so we report empirical run times rather than multiplying baseline costs by a fixed factor.

\begin{table}[H]
\centering
\begin{tabular}{@{}lcccl@{}}
\toprule
Experiment & Runs & Seeds & Wall time / run & Total GPU-hours \\
\midrule
Modular Arithmetic & 400 & 5 & 0.2--41 min & 73 \\
CIFAR-10 & 96 & 3 & 0.9--2.0 h & 113 \\
SST-5 & 168 & 3 & 0.35--3.82 h & 290 \\
MNIST & 110 & 3 & -- & 2 \\
IMDb & 261 & 3 & 0.5--1.2 h & 113 \\
\bottomrule
\end{tabular}
\caption{
We report the number of total runs, seeds, wall-time per run, and total GPU-hours to recreate all results on each task independently.}
\label{table:compute}
\end{table}

Computing sharpness measurements adds additional costs beyond training. Han-style final-layer flatness metrics are inexpensive relative to training, while Hessian maximum-eigenvalue estimates require repeated Hessian-vector products. For the modular-arithmetic tables, final $\lambda_{\max}$ values were computed with 50 iterations of the Lanczos method to the empirical training-loss Hessian with respect to all trainable parameters at the final checkpoint; this was run post hoc and is therefore not included in the training-time estimates in Table~\ref{table:compute}.

\end{document}